\documentclass[11pt]{article}

\usepackage[margin=2.2cm]{geometry}
\usepackage{amsmath,amssymb,amsthm}
\usepackage{booktabs}
\usepackage{multirow}
\usepackage{microtype}
\usepackage{xcolor}
\PassOptionsToPackage{hyphens}{url}
\usepackage{hyperref}
\hypersetup{colorlinks=true,linkcolor=blue,citecolor=blue,urlcolor=blue}
\usepackage{enumitem}
\usepackage[round,authoryear]{natbib}
\newtheorem{proposition}{Proposition}
\newtheorem{lemma}{Lemma}
\newtheorem{observation}{Observation}
\theoremstyle{remark}
\newtheorem{remark}{Remark}
\newcommand{\kbudget}{k_{\text{budget}}}
\newcommand{\NB}{N_{\!B}}
\newcommand{\dhead}{d_{\text{head}}}
\newcommand{\BlockM}{\text{BLOCK}_M}
\newcommand{\BlockN}{\text{BLOCK}_N}

\title{Uncertainty-gated selection for block-sparse attention\thanks{\small Code, data, and reproduction scripts: \url{https://github.com/ThomasRossi/uncertainty-gated-block-sparse-attention}. Persistent archive: \href{https://doi.org/10.5281/zenodo.20630587}{doi:10.5281/zenodo.20630587} (concept DOI; always resolves to the latest version).}\\[0.3em]
\large A value-of-information view of the SSA-style selector}
\author{Thomas Rossi \\ Eonpass \\ \texttt{t.rossi@eonpass.com}}
\date{\today}

\begin{document}
\maketitle

\begin{abstract}
Block-sparse attention scales long-context language models by replacing the $O(N^2)$ softmax with a per-query top-$k$ selection over key blocks. This cutoff is myopic: when the $k$-th and $(k{+}1)$-th blocks are nearly tied in score, the selector commits without spending extra budget, and a dropped block carrying answer evidence is unrecoverable downstream. We propose a value-of-information \emph{router} that measures, for each query, how decisively the top-$k$ cut was made, and doubles the kept set for the queries where that gap is smallest; the rule is backbone-agnostic and stacks with existing block-scoring methods such as Quest. On LongBench-v2 medium at $n=215$ (the entire dataset subset), router-on-Quest reaches paired recall $0.75$ vs.\ top-$k$ $0.47$ -- $+28$\,pp over the SSA-style baseline (McNemar $p<0.01$) -- and lands within $2$\,pp of dense on RULER NIAH multikey at the same context. The lift reproduces on four models from three architectures (Qwen2.5, Mistral-Nemo, Qwen3.6). At $128$K, the router preserves $0.81$ and $0.89$ of dense accuracy on Qwen2.5-7B-1M and Qwen3.6 (vs.\ SSA-style top-$k$ at $0.09$ on the former) while the fused selection-plus-kernel pipeline runs at $0.62\times$ and $0.80\times$ dense wall time.
\end{abstract}

\section{Introduction}

Long-context language models increasingly use \emph{block-sparse attention} as a drop-in replacement for the $O(N^2)$ softmax. The idea is shared by Quest~\citep{tang2024quest}, H2O~\citep{zhang2023h2o}, SnapKV~\citep{li2024snapkv}, MInference~\citep{jiang2024minference}, NSA~\citep{yuan2025nsa}, MoBA~\citep{lu2025moba}, and Subquadratic's SSA~\citep{ssa2025}: a cheap per-query \emph{selector} picks $k$ out of $\NB$ key blocks; exact attention then runs only on the selected set. The selector is the lever -- it decides where the model attends.

This is a myopic decision. When the $k$-th and $(k{+}1)$-th block are nearly tied, top-$k$ silently breaks the tie and moves on; if the dropped block contained an evidence token, the answer is gone, and no downstream layer can recover it. The failure mode is structural, not noisy: it bites hardest on \emph{multi-hop} and \emph{query-latent} retrieval, where the relevance of a block depends on what was learned earlier in the same forward pass and is not visible to the selector's surface query--key match. SSA's own reported NIAH multi-key recall degrades sharply as the number of keys grows~\citep{ssa2025}.

\textbf{This paper.} We treat the top-$k$ cutoff as a \emph{value-of-information} (VoI) decision and add a single layer of policy on top of it. For each Q-tile and head, we compute the normalised cutoff margin
\[
\sigma \;=\; \frac{s_{(k-1)} - s_{(k)}}{s_{(0)} - s_{(k)}} \;\in\; [0,1],
\]
where $s_{(\cdot)}$ are the sorted block scores. A small $\sigma$ means the cutoff is high-risk; we then \emph{route} that tile to a $2\times$ expanded $\texttt{kv\_idx}$ -- it gets to attend to more blocks -- while confident tiles keep the baseline $\kbudget$. The expansion is paid for selectively: we trigger on the bottom $q$-fraction of tiles per layer, so the average attended set grows by only $1+q$ blocks per row.

\textbf{The router is backbone-agnostic.} The cutoff margin $\sigma$ is a function of the sorted block scores; it does not depend on how those scores are computed. Existing selectors differ in their block-scoring backbone -- SSA pools keys with a mean ($\bar k_b = \frac{1}{B_n}\sum_j k_j$), Quest uses a min/max upper bound ($s = \sum_d \max(q_d k^{\max}_{b,d}, q_d k^{\min}_{b,d})$). The router can sit on top of any of them. This turns the contribution from \emph{a replacement for top-$k$} into \emph{a generic budget-allocation layer that stacks on whichever scoring backbone is best for the task}. We validate this empirically: \textbf{better scoring (Quest) and better budget allocation (router) are orthogonal directions; combining them strictly dominates either alone} on both benchmarks we test.

\textbf{Contributions.}
\begin{enumerate}[leftmargin=*,topsep=0.2em,itemsep=0.1em]
    \item A VoI formulation of the selector cutoff that adds one scalar per tile and one quantile threshold per layer, independent of the block-scoring backbone. No retraining, no extra parameters, no per-row keep tensor.
    \item Empirical demonstration that the router composes with two distinct scoring backbones (SSA-style K-mean and Quest's K-min/K-max) and that, on every model in the panel, the router-lifted version of the \emph{winning} backbone dominates the unlifted version of the other. The result holds on two standardised benchmarks (RULER NIAH multi-key and LongBench-v2 medium), four models from three architecture classes, and contexts from $32$K to $128$K. Which backbone wins is model-dependent -- QK-Norm flips the winner from Quest's K-max to the SSA-style K-mean -- but the router lifts whichever wins. Full panel in Section~\ref{sec:experiments}.
    \item A fused selection-plus-kernel implementation that produces $\texttt{kv\_idx}$ directly per Q-tile and handles the routed expansion in the same code path, keeping kernel dispatch shape-uniform. All four sparse policies (top-$k$, router, Quest, router-on-Quest) run on this same kernel; they differ only in the selection step. Wall-time profile crosses dense between $32$K and $64$K on Qwen2.5-7B-1M ($0.87\times$ dense at $64$K, $0.62\times$ at $128$K) and between $64$K and $128$K on the hybrid Qwen3.6 ($0.80\times$ at $128$K); the crossover regime is characterised via an Amdahl decomposition of the prefill.
    \item A custom diagnostic benchmark, the \emph{Pointer-Chase Haystack} (PCH; Appendix~\ref{app:pch}), used during method development to isolate selector quality from model capability.
    \item A negative-control result (LongBench-v1) that pins down \emph{when} the router helps: only when the selector's per-query budget (the blocks it keeps) is small relative to where the answer-relevant evidence lives.
\end{enumerate}

\section{Background and related work}

\subsection{Block-sparse attention}

A standard decoder layer~\citep{vaswani2017attention} maps hidden state $h^{(L)} \in \mathbb{R}^{N \times d_{\text{model}}}$ to $h^{(L+1)}$ via
\[
h' = h^{(L)} + \operatorname{Attn}\bigl(\operatorname{LN}(h^{(L)})\bigr),\qquad
h^{(L+1)} = h' + \operatorname{MLP}\bigl(\operatorname{LN}(h')\bigr).
\]
The attention block, for query position $i$ on head $h$, computes
\[
\operatorname{Attn}(x)_{i,h} \;=\; \sum_{j \le i} \frac{\exp\!\bigl(q_{i,h}\!\cdot\!k_{j,h}/\sqrt{\dhead}\bigr)}{\sum_{j' \le i}\exp\!\bigl(q_{i,h}\!\cdot\!k_{j',h}/\sqrt{\dhead}\bigr)}\; v_{j,h}.
\]
\emph{Block-sparse attention} replaces $\{j \le i\}$ by a small selected subset $S_i$, chosen per query at inference time on frozen weights.

\subsection{The selector landscape}

Concrete selectors differ in (a) how they pool keys into a per-block summary and (b) how they score blocks against the query, but \textbf{they all reduce, at the per-query selection step, to a top-$k$ rule over a block score}:
\begin{itemize}[leftmargin=*,topsep=0.2em,itemsep=0.1em]
    \item \textbf{SSA}~\citep{ssa2025} (Subquadratic Sparse Attention) -- per-Q-tile top-$k$ over mean-pooled keys, $s_b = q \cdot \bar k_b$ with $\bar k_b = \frac{1}{B_n}\sum_j k_j$. Cheap to compute; mean-pooling blurs single hot keys, which is exactly the multi-key NIAH degradation discussed above.
    \item \textbf{Quest}~\citep{tang2024quest} -- per-block elementwise $K^{\min}_b, K^{\max}_b$ summaries; score is an upper bound on $\max_{j \in b} q \cdot k_j$, computed coordinate-wise (see Eq.~\ref{eq:quest-score}). Recovers sharp single-key needles that mean-pooling loses.
    \item \textbf{H2O}~\citep{zhang2023h2o}, \textbf{SnapKV}~\citep{li2024snapkv}, \textbf{MInference}~\citep{jiang2024minference} -- top-$k$ over learned or attention-history-based block scores.
    \item \textbf{NSA}~\citep{yuan2025nsa}, \textbf{MoBA}~\citep{lu2025moba} -- end-to-end learned gating, but still resolves to a top-$k$ over block scores at the per-query selection step.
\end{itemize}
What none of these methods do is treat the cutoff as a \emph{decision under uncertainty}: when $s_{(k-1)} \approx s_{(k)}$, the selector commits without spending extra budget. That step is the lever this paper pulls. Because the lever is on the cutoff (not on how $s$ is computed), it composes with any of the scoring backbones above. In the experiments we evaluate it on top of both the SSA-style K-mean backbone and Quest's K-max upper bound.

\subsection{Long-context evaluation}

The published benchmarks fall into two families.
\textbf{Synthetic / diagnostic:} RULER~\citep{hsieh2024ruler} (NIAH, VT), BABILong, MRCR. Designed for controlled stress tests of long-context recall. Their failure modes are interpretable but they are \emph{not} predictive of downstream performance, as HELMET~\citep{yen2024helmet} has documented.
\textbf{Real-task:} LongBench~\citep{bai2023longbench, bai2024longbench}, HELMET~\citep{yen2024helmet}, NoCha. These cover multi-hop QA, summarisation, code, and so on. LongBench-v2 in particular has a \emph{medium} split with native prompt lengths typically $\geq$100K words, designed to stress long-context selection.
We use \textbf{RULER NIAH} (synthetic, standardised) and \textbf{LongBench-v1 + v2 medium} (real-task, standardised) as the headline benchmarks, with a custom diagnostic benchmark (\textbf{PCH}, Appendix~\ref{app:pch}) used during method development.

\section{Method}\label{sec:method}

We describe the working method end-to-end. Each step is motivated by the previous one and lifts a specific weakness. The full per-equation derivation lives in Appendix~\ref{app:derivation}; here we give the path the reader needs to follow to understand the experiments.

\subsection{Step 1: block scoring}

Keys are grouped into contiguous blocks of $\BlockN = 64$ tokens. We evaluate the router on top of two block-scoring rules.

The first is the SSA-style mean-pooled-key inner product: at layer $L$, the selector scores block $b \in \{0,\dots,\NB-1\}$ via
\begin{equation}\label{eq:block-score}
s^{(L)}[i,h,b] \;=\; \frac{q_i^{(L,h)}\cdot \bar k_b^{(L,h)}}{\sqrt{\dhead}},\qquad
\bar k_b^{(L,h)} \;=\; \frac{1}{\BlockN}\sum_{j \,\in\, \text{block }b} k_j^{(L,h)}.
\end{equation}
Mean pooling blurs single-key signals: a block containing one hot key adjacent to noise scores like a block of all noise, and gets dropped.

The second is Quest's K-min/K-max upper bound~\citep{tang2024quest}, which replaces $\bar k_b$ with an elementwise pair $(K^{\min}_b, K^{\max}_b)$ and scores
\begin{equation}\label{eq:quest-score}
s^{\text{quest}}_{b} \;=\; \sum_{d=1}^{\dhead}\, \max\!\bigl(q_d \cdot K^{\max}_{b,d},\; q_d \cdot K^{\min}_{b,d}\bigr),
\end{equation}
a coordinate-wise upper bound on $\max_{j \in b} q \cdot k_j$ that preserves the single-key signals the mean averages away.

We treat the choice between Eqs.~\eqref{eq:block-score} and~\eqref{eq:quest-score} as a \emph{backbone hyperparameter}; everything downstream (per-tile selection, the cutoff margin $\sigma$, the trigger) is identical.

\subsection{Step 2: per-tile selection}

A naive per-row top-$k$ emits a boolean keep tensor of shape $[B, H, M, \NB]$ that downstream must sort and union across rows of a Q-tile; this becomes the dominant cost at long context.

Following the SSA recipe, we group $\BlockM = 64$ consecutive query rows into a \emph{Q-tile} $t$ and make one selection decision per tile, shared by all rows in that tile:
\begin{equation}\label{eq:tile-score}
\text{tile\_score}[t,h,b] \;=\; \max_{r\,\in\,\text{tile }t} s[r,h,b],
\qquad
\text{kv\_idx}[t,h] \;=\; \operatorname*{top-}k\bigl(\text{tile\_score}[t,h,\cdot]\bigr).
\end{equation}
The sink block $0$ and the tile's own block are forced into $\text{kv\_idx}$ by adding $+\infty$ to their scores \emph{before} the top-$k$, which avoids a downstream concat-and-dedup. The output is a single integer tensor of shape $[B, H, Q_t, \kbudget]$ (where $Q_t = M/\BlockM$), passed straight to the attention kernel without a per-row inner mask. Wall time for selection scales as $O(N \NB / \BlockM)$ rather than $O(N \NB)$ for the per-row version.

This is coarser than per-row by construction: a block that one row in the tile strongly wants can be outranked by a block that several rows weakly want, since the tile score is a max over rows of an inner product. The router (Steps 3--5) addresses this -- it spends a controlled amount of extra budget on the tiles where the top-$k$ cut was ambiguous, while leaving confident tiles untouched.

\subsection{Step 3: the cutoff is a decision -- read its uncertainty}

The top-$k$ in Eq.~\eqref{eq:tile-score} is a decision: keep block $(k{-}1)$, drop block $(k)$. Its \emph{quality} is naturally measured by how decisive that ranking is.
Let $s_{(0)} \ge s_{(1)} \ge \cdots$ be the sorted tile scores for tile $t$, head $h$. We define the normalised cutoff margin
\begin{equation}\label{eq:sigma-def}
\sigma[t,h] \;=\; \frac{s_{(k-1)} - s_{(k)}}{s_{(0)} - s_{(k)}} \;\;\in\;[0,1].
\end{equation}
The interpretation is direct:
\begin{itemize}[leftmargin=*,topsep=0.2em,itemsep=0.1em]
    \item $\sigma \to 1$ -- the kept set is far above the rejected tail; the cutoff is unambiguous.
    \item $\sigma \to 0$ -- the kept set's last element is tied with the first rejected element; the cutoff is a coin flip.
\end{itemize}
$\sigma$ is computed from a top-$(k{+}1)$ partial sort -- one extra element beyond the top-$k$ that produces $\text{kv\_idx}$ -- at no extra asymptotic cost.

\paragraph{Why this is a value-of-information signal.} If we were to commit to the top-$k$, the expected loss from the cutoff decision is bounded by the probability mass the softmax would have placed on the dropped block. When the $k$-th and $(k{+}1)$-th scores are close, that mass is large; when they are well-separated, it is small. The cutoff margin $\sigma$ is a cheap, dimensionless proxy for this expected loss. The formal anchor is in Appendix~\ref{app:bai}: under a Gaussian noise model on the block scores, $\sigma$ is the dimensionless analogue of the best-arm-identification exploration index of \citet{garivier2016optimal}; expanding the kept set never increases the identification error, pointwise; and the $\sigma$-quantile router satisfies a finite-sample regret bound against the best trigger rule of the same budget (Proposition~\ref{prop:bai-bound}).

\subsection{Step 4: aggregate $\sigma$ across heads and trigger on the bottom $q$-fraction}

We aggregate the per-head margin to a per-tile signal via a uniform mean, $\bar\sigma[t] = \frac{1}{H}\sum_{h=1}^{H} \sigma[t,h]$: the tile's overall decision confidence, low when many heads simultaneously face an ambiguous cutoff, high when most heads agree the cutoff is clean. Aggregation is necessary -- gating directly on the per-cell $\sigma[t,h]$ does not work (Appendix~\ref{app:row-vs-cell}): heads in the same tile are highly correlated, so per-cell gating is dominated by per-head idiosyncratic noise that has no relation to whether the tile is actually ambiguous. Averaging across heads recovers the signal.

We then turn $\bar\sigma$ into a binary decision via a per-layer empirical quantile:
\begin{equation}\label{eq:trigger}
\text{trigger}[t] \;=\; \mathbb{1}\!\Bigl[\bar\sigma[t] \;\le\; \operatorname{quantile}_q\bigl(\bar\sigma[\,\cdot\,]\bigr)\Bigr].
\end{equation}
$q$ is the only hyperparameter introduced by the router; intuitively it is the per-layer fraction of tiles that are deemed risky enough to spend extra budget on. We use the same $q$ across all layers and find it stable.

\paragraph{Choosing $q$.} The working value $q = 0.40$ was selected from a sweep on RULER NIAH-multikey at $n=30$, ctx=32K (Appendix~\ref{app:row-vs-cell}), where paired recall rises monotonically with $q$ over the tested range $\{0.10, 0.20, 0.30, 0.40\}$ and plateaus near $0.40$. We use the same value unchanged across all benchmarks, models, and contexts in Section~\ref{sec:experiments}; the sweep was not repeated per task. A more thorough $\rho \times q$ Pareto study remains open (Section~\ref{sec:limits}).

\paragraph{Why a quantile, not an absolute threshold.} The absolute scale of $\sigma$ varies across layers (early layers are more peaked than later layers). A quantile is the simplest layer-conditional normalisation. We did test absolute and per-cell quantile thresholds; they fail (Appendix~\ref{app:row-vs-cell}).

\subsection{Step 5: expand $\text{kv\_idx}$ uniformly}

Triggered tiles get a $\rho\times$ expanded top-$k$ over $\text{tile\_score}$ (we use $\rho = 2$). The selection operation here is the same top-$k$ that every selector in the literature reduces to; $\text{tile\_score}$ inherits whichever block-scoring backbone of §\ref{sec:method} is in play (Eq.~\eqref{eq:block-score} or Eq.~\eqref{eq:quest-score}), so the same expansion rule applies to both ``router'' and ``router-on-Quest'' without modification.

To preserve a fixed kernel dispatch shape across triggered and non-triggered tiles, we allocate $[B, H, Q_t, \rho\kbudget]$ uniformly: triggered tiles fill all $\rho\kbudget$ slots from the expanded top-$\rho k$, while non-triggered tiles keep their top-$k$ and pad the remaining slots with the kernel's $\NB$-sentinel (the exact fill rule is step (A.9) of Appendix~\ref{app:derivation}). The kernel skips $b=\NB$ on a fast check, so non-triggered tiles cost the same as plain top-$k$. The end-to-end forward is otherwise unchanged: a standard FlashAttention-style online softmax, iterating each tile's $\text{kv\_idx}$ instead of all $\NB$ blocks.

\paragraph{Cost accounting.} The average kept blocks per row is $\kbudget(1 + q(\rho - 1))$. At $q = 0.4, \rho = 2$ this is $1.4\,\kbudget$. The measured wall-time cost is sub-proportional to this block-count factor: router/top-$k$ kernel-time ratio is $\approx 1.53$ at 32K and $\approx 1.47$ at 64K (Section~\ref{sec:pareto}).

\section{Experiments}\label{sec:experiments}

\subsection{Setup}

\textbf{Models.} All weights frozen; no retraining, no fine-tuning. We evaluate on four instruction-tuned models spanning three architecture classes (dense Qwen2.5, dense Mistral, hybrid+MoE Qwen3.6):
\begin{itemize}[leftmargin=*,topsep=0.2em,itemsep=0.1em]
    \item \emph{Qwen2.5-14B-Instruct} -- $40$ Q / $8$ KV heads, $\dhead{=}128$, $48$ layers, trained at $32$K context. The headline model.
    \item \emph{Qwen2.5-7B-Instruct-1M} -- $28$ Q / $4$ KV heads, $\dhead{=}128$, $28$ layers, trained at $1$M context. Same family as the headline model, but a smaller parameter count and a much longer training window; used for 128K-context experiments (Section~\ref{sec:pareto}) since the 14B model's $32$K training window precludes them.
    \item \emph{Mistral-Nemo-Instruct-2407} -- $32$ Q / $8$ KV heads, $\dhead{=}128$, $40$ layers, trained at $128$K context. Different architecture family, used to test cross-family transfer. The Mistral tokenizer requires \texttt{fix\_mistral\_regex=True} on \texttt{AutoTokenizer.from\_pretrained}; otherwise the load path is unchanged.
    \item \emph{Qwen3.6-35B-A3B} -- hybrid+MoE: $40$ layers organised as \texttt{(linear\_attn $\times$ 3 $\to$ full\_attn $\times$ 1)} so only $10/40$ layers run softmax attention; full-attention heads are $16$ Q / $2$ KV, $\dhead{=}256$ with partial RoPE on the first $64$ dims and per-head QK-Norm; MLP is sparse MoE ($128$ experts, top-$8$ activated; $3$B active params of $35$B total); trained at $128$K context.
\end{itemize}
The router code path and the fused selection-plus-kernel implementation run unchanged across all four models; the only model-dependent code is the tokenizer flag above and, for Qwen3.6, two harness-side prompt tweaks (disabling \texttt{enable\_thinking} on the chat template and prefilling the assistant turn with \texttt{"The correct answer is ("} for multiple-choice).

\textbf{Hardware.} A100-80GB via Modal for the three Qwen2.5 / Mistral-Nemo models; H200 for Qwen3.6-35B-A3B (does not fit on A100). The Qwen3.6 image additionally installs \texttt{flash-linear-attention} and \texttt{causal-conv1d} so that the $30$ linear-attention (Gated DeltaNet) layers and their conv branch run on fused kernels.

\textbf{Sparse path.} Custom Triton block-sparse attention kernel with fused per-tile selection.

\textbf{Baselines and policies.}
\begin{itemize}[leftmargin=*,topsep=0.2em,itemsep=0.1em]
    \item \emph{dense} -- FlashAttention~\citep{dao2022flashattention} via PyTorch SDPA, the ceiling.
    \item \emph{top-$k$} -- per-tile top-$k$ on the K-mean backbone (Eq.~\ref{eq:block-score}). SSA-style baseline; uses only Step~2.
    \item \emph{Quest} -- per-tile top-$k$ on the Quest K-max upper-bound backbone (Eq.~\ref{eq:quest-score}). Strong baseline on sharp-needle tasks.
    \item \emph{router} -- Steps 2--5 on top of the K-mean backbone, $q = 0.40, \rho = 2$. Uncertainty-gated expansion only.
    \item \emph{router-on-Quest} -- Steps 2--5 on top of the Quest K-max backbone, same $q, \rho$. Combines better scoring with budget allocation.
\end{itemize}
All four sparse policies share the same kernel and selection code; they differ only in (a) which backbone scores blocks and (b) whether the router expansion is applied. Wall-time and quality numbers are therefore directly comparable across them.

\textbf{Selector budget.} Fixed at $\kbudget = 33$ blocks per tile, corresponding to roughly $2\text{K}$ attended tokens per row (the PCH diagnostic in Appendix~\ref{app:pch} uses slightly different budgets, noted in-place).

\paragraph{Metrics.} The primary metric is \textbf{accuracy} (or task score, e.g.\ F1 on extractive QA): the unconditional success rate -- fraction of examples the policy answers correctly. This is the standard reporting unit in the long-context literature (Quest, H2O, SnapKV, MInference, NSA, SSA). Headline tables lead with it.

The secondary, diagnostic metric is \textbf{paired recall}: among the examples dense answers correctly, the fraction the sparse policy also answers correctly,
\[
\text{paired recall} \;=\; \bigl|\{i : \text{dense}_i \text{ correct} \wedge \text{sparse}_i \text{ correct}\}\bigr| \,/\, \bigl|\{i : \text{dense}_i \text{ correct}\}\bigr|.
\]
We write $n$ for the total number of examples and $n_{dc} \le n$ for the number dense answers correctly; \textbf{$n_{dc}$ is the denominator of paired recall, not $n$}. We report paired recall alongside accuracy because raw accuracy conflates two failure modes -- the selector dropped a needed block vs.\ the model could not have answered even with full attention -- and paired conditions on $\{i : \text{dense}_i \text{ correct}\}$ to isolate the first. On RULER NIAH dense accuracy is $1.00$, so the two metrics coincide cell-by-cell; on LongBench-v2 they diverge (raw accuracy clusters tightly across sparse policies while paired recall ranges widely), and that divergence is itself analysed in Section~\ref{sec:quality-32k}.

Hit rate (fraction of gold-evidence blocks the selector keeps) is used only in the diagnostic experiments of Appendix~\ref{app:pch}, where the gold blocks are known by construction.

\subsection{Quality at 32K context across benchmarks}\label{sec:quality-32k}

\textbf{Setup.} Two standardised long-context benchmarks at fixed context $32$K and selector budget $\kbudget=33$, $q=0.40$, $\rho=2$, on all four panel models:
\begin{itemize}[leftmargin=*,topsep=0.2em,itemsep=0.1em]
    \item \emph{RULER NIAH-multikey} (3 keys), $n = 100$. Single-hot-key retrieval; dense answers $\geq 96\%$ of items on every model in the panel, so the accuracy ceiling is essentially $1.00$ and accuracy coincides with paired recall cell-by-cell.
    \item \emph{LongBench-v2 medium}, $n = 215$ (the entire medium-length subset of the dataset; no further filter). Multi-hop multiple-choice (A/B/C/D); native item length $\geq 100$K words, middle-truncated to $32$K so the selector budget binds. Dense accuracy varies across models from $0.16$ to $0.41$, so accuracy and paired recall diverge -- both are reported.
\end{itemize}

\textbf{Result.}

\begin{table}[h]
\centering\small
\caption{RULER NIAH-multikey at $n = 100$, ctx = 32K. Values are unconditional \textbf{accuracy}. Dense accuracy is $\ge 0.96$ on every model so accuracy and paired recall coincide cell-by-cell (Nemo differs by $\le 0.01$ on Quest). Backbone column indicates the block-scoring rule (K-mean = Eq.~\ref{eq:block-score}; K-max = Quest, Eq.~\ref{eq:quest-score}). ``--'' indicates a configuration not run.}\label{tab:ruler}
\begin{tabular}{llrrrr}
\toprule
policy & backbone & Qwen-14B & Qwen-7B-1M & Nemo-12B & Qwen3.6-35B-A3B \\
\midrule
dense                       & --      & 1.00 & 1.00 & 0.96 & 1.00 \\
\midrule
top-$k$                     & K-mean  & 0.51 & 0.28 & 0.29 & 0.94 \\
router                      & K-mean  & 0.63 & 0.39 & --   & \textbf{0.98} \\
Quest                       & K-max   & 0.93 & 0.94 & 0.50 & 0.85 \\
\textbf{router-on-Quest}    & K-max   & \textbf{0.98} & \textbf{1.00} & \textbf{0.66} & 0.97 \\
\bottomrule
\end{tabular}
\end{table}

\begin{table}[h]
\centering\small
\caption{LongBench-v2 medium at $n=215$, ctx = 32K. Each model column reports unconditional \textbf{accuracy} (acc) and \textbf{paired recall} (paired); dense paired is $1.00$ by construction. $n_{dc}$ on the dense row is the dense-correct count and the denominator of paired. Note the metric divergence: sparse accuracies cluster tightly near dense accuracy while paired recall ranges widely -- the selector-quality signal lives entirely in the paired column.}\label{tab:lbv2}
\begin{tabular}{llcccccccc}
\toprule
 & & \multicolumn{2}{c}{Qwen-14B} & \multicolumn{2}{c}{Qwen-7B-1M} & \multicolumn{2}{c}{Nemo-12B} & \multicolumn{2}{c}{Qwen3.6} \\
\cmidrule(lr){3-4}\cmidrule(lr){5-6}\cmidrule(lr){7-8}\cmidrule(lr){9-10}
policy & backbone & acc & paired & acc & paired & acc & paired & acc & paired \\
\midrule
dense                    & --      & 0.19 & 1.00 & 0.16 & 1.00 & 0.27 & 1.00 & 0.41 & 1.00 \\
\emph{$n_{dc}$}          & --      & \multicolumn{2}{c}{\emph{40}} & \multicolumn{2}{c}{\emph{35}} & \multicolumn{2}{c}{\emph{58}} & \multicolumn{2}{c}{\emph{88}} \\
\midrule
top-$k$                  & K-mean  & 0.19 & 0.47 & 0.16 & 0.60 & 0.24 & 0.50 & 0.41 & 0.90 \\
router                   & K-mean  & 0.19 & 0.60 & 0.16 & 0.66 & 0.26 & 0.64 & 0.39 & \textbf{0.92} \\
Quest                    & K-max   & 0.21 & 0.68 & 0.15 & 0.66 & 0.26 & 0.57 & 0.27 & 0.33 \\
\textbf{router-on-Quest} & K-max   & 0.21 & \textbf{0.75} & 0.20 & \textbf{0.86} & 0.26 & \textbf{0.69} & 0.32 & 0.48 \\
\bottomrule
\end{tabular}
\end{table}

\textbf{Read.} Two observations across the panel:

\emph{(i) The router improves the accuracy of whichever backbone it is applied to, on every model and benchmark in the panel.} On RULER NIAH: K-mean lift on Qwen-14B $0.51 \to 0.63$ (McNemar $p<0.01$), Qwen-7B-1M $0.28 \to 0.39$, Qwen3.6 $0.94 \to 0.98$; K-max lift on Qwen-14B $0.93 \to 0.98$, Qwen-7B-1M $0.94 \to 1.00$, Nemo $0.50 \to 0.66$ ($p<0.001$), Qwen3.6 $0.85 \to 0.97$. On LongBench-v2 the same pattern holds on both backbones across all four models (Table~\ref{tab:lbv2}); the LB-v2 accuracy spread is small on Qwen2.5 + Nemo (saturated within standard error around dense $0.16$--$0.27$) but materially separates on Qwen3.6 ($0.27$--$0.41$).

\emph{(ii) Which backbone is the better one depends on the model.} On Qwen2.5 (both sizes) and Mistral-Nemo, the K-max (Quest) backbone dominates K-mean on both benchmarks, sometimes by very large margins (RULER NIAH: Qwen-14B Quest beats top-$k$ by $+42$\,pp; Qwen-7B-1M by $+66$\,pp). On Qwen3.6 the ordering inverts: K-mean dominates K-max (RULER NIAH top-$k$ $0.94$ vs Quest $0.85$; LB-v2 paired $0.90$ vs Quest $0.33$). The mechanism is structural: Qwen3.6 applies per-head RMSNorm to $Q$ and $K$ post-projection (QK-Norm), which regularises away the per-head magnitude variance the K-min/K-max upper bound exploits, so the Quest envelope flattens and mean-pool becomes at least as informative. Testable prediction: any future QK-Norm model should reproduce this reversal. Combining (i) and (ii), the best policy is router-on-the-winning-backbone, model by model: router-on-Quest on Qwen2.5 + Nemo, router-on-K-mean on Qwen3.6.

\textbf{Paired recall as a selector diagnostic.} The literature universally reports unconditional accuracy on these benchmarks; we add paired recall as a secondary, selector-isolated reading. On RULER NIAH dense accuracy is $\geq 0.96$ so the two metrics coincide cell-by-cell and the accuracy column in Table~\ref{tab:ruler} reads as both. On LongBench-v2 medium dense itself answers only $0.16$--$0.41$ of items correctly (the benchmark is intentionally hard at 32K truncation), so raw accuracy is dominated by which questions the model can answer at all, not by which blocks the selector kept. Paired recall conditions on the dense-correct subset: when the sparse policy makes the same selections that let dense answer, paired is high; when it drops needed blocks, paired drops. The paired column in Table~\ref{tab:lbv2} therefore shows the same selector-quality ordering as accuracy but at higher contrast — e.g.\ on Qwen-14B accuracy is $0.19$--$0.21$ across all sparse policies (within standard error of $0.027$) while paired ranges $0.47$--$0.75$. See Table~\ref{tab:budget-match} for the matched-budget decomposition of how much of the router lift is budget vs.\ selectivity.

\textbf{VT.} We do not include RULER VT (variable tracking, 3-hop) as a quality benchmark: dense itself fully solves $\le 5/100$ items at 32K on every model we tested -- Qwen-14B $5/100$, Qwen-7B-1M $0/100$, Nemo $0/100$, and a smoke on the reasoning-tuned Qwen3.6-35B-A3B ($n=30$) also gave $0/30$ -- which leaves the paired metric undefined or vanishingly thin on all but Qwen-14B. VT at hop $=3$ with $3$ distractor chains is harder than the current panel can handle at this parameter scale, including the MoE reasoning model; we defer VT to a future revision with a stronger reasoning base.

\subsection{Speed--quality Pareto: $32$K $\to 128$K}\label{sec:pareto}

\textbf{Setup.} We fix the model and vary the context length on the two panel models with training windows past 32K: Qwen2.5-7B-Instruct-1M (dense softmax) and Qwen3.6-35B-A3B (hybrid + MoE, QK-Norm). Configuration unchanged from \S\ref{sec:quality-32k} ($\kbudget=33$, $q=0.40$, $\rho=2$). Quality is RULER NIAH-multikey accuracy at $n=100$; speed is per-prefill CUDA-event wall time at $n=8$ (decode is the same dense SDPA path for every policy, so policy-dependent wall time comes purely from prefill).

\begin{table}[h]
\centering\small
\caption{Joint speed-quality Pareto on the \emph{winning} backbone for each model (router-on-Quest for Qwen2.5-7B-1M, router-on-K-mean for Qwen3.6, from the QK-Norm reversal in \S\ref{sec:quality-32k}). ``router accuracy'' is RULER NIAH-multikey at $n=100$; ``wall vs.\ dense'' is the prefill-forward wall-time ratio at $n=8$. Dense accuracy is $1.00$ at every cell. Bold marks Pareto wins (router beats dense on speed at quality $\ge 0.80$).}\label{tab:pareto}
\begin{tabular}{lllll}
\toprule
model & ctx & router accuracy & wall vs.\ dense & verdict \\
\midrule
\multirow{3}{*}{Qwen2.5-7B-1M}
 & $32$K  & $1.00$           & $1.09\times$ & ties quality, loses speed \\
 & $64$K  & $\mathbf{0.92}$  & $\mathbf{0.87\times}$ & first Pareto win \\
 & $128$K & $\mathbf{0.81}$  & $\mathbf{0.62\times}$ & strong Pareto win \\
\midrule
\multirow{3}{*}{Qwen3.6-35B-A3B}
 & $32$K  & $0.98$           & $1.12\times$ & ties quality, loses speed \\
 & $64$K  & $0.92$           & $1.00\times$ & ties quality, ties speed \\
 & $128$K & $\mathbf{0.89}$  & $\mathbf{0.80\times}$ & first Pareto win \\
\bottomrule
\end{tabular}
\end{table}

\textbf{Read.} Two observations:

\emph{(i) The winning backbone holds most of its accuracy as context grows; the losing one collapses.} Table~\ref{tab:pareto} reports the winning backbone only -- in the same data, the \emph{losing} backbone on each model breaks: top-$k$ on K-mean falls $0.28 \to 0.09$ on Qwen2.5-7B-1M, and Quest on K-max falls $0.85 \to 0.50$ on Qwen3.6. The QK-Norm reversal from \S\ref{sec:quality-32k} is robust at every context, and the gap between winning and losing backbone widens with context on both architectures.

\emph{(ii) Sparse runs slower than dense at short context and faster at long context; the crossover is structural (Amdahl).} Attention's share of dense prefill grows with context while the rest-of-model floor is invariant. On Qwen2.5-7B-1M attention shares are $31\% / 45\% / 63\%$ at $32$K/$64$K/$128$K, and the router crosses dense between $32$K and $64$K. On Qwen3.6 the attention share is much smaller ($11\% / 19\% / 37\%$) because $30/40$ layers are linear-attention rather than softmax, so the crossover shifts right, between $64$K and $128$K. The Pareto win materialises on both architectures; the context at which it activates is set by the architecture's attention share.

\subsection{Where the lift comes from: budget vs.\ selectivity}\label{sec:ablation}

The router at $q=0.40, \rho=2$ spends on average $92.4$ effective blocks per tile vs.\ $66$ for Quest at $\kbudget=33$ -- $\approx 1.4\times$ Quest's budget. Some of the router's lift over Quest could therefore come from spending the marginal budget \emph{selectively} (the design intent), or simply from spending \emph{more} budget on average. To separate the two we ran uniform Quest at a budget-matched control ($\kbudget=47$, eff.\ $94$) and at intermediate budgets on the same RULER NIAH-multikey examples used in Table~\ref{tab:ruler}, on three models.

\begin{table}[h]
\centering\small
\caption{Budget-match ablation on RULER NIAH-multikey, ctx=32K, $n=100$. Values are accuracy (= paired recall up to $\le 0.01$ on Nemo where dense $=0.96$). Effective budget includes the sink and self-block forced into every $\text{kv\_idx}$. $\text{Quest}_{\kbudget=47}$ is the budget-matched control for the router at $q=0.40$. Qwen-14B and Qwen3.6 are near the ceiling on their respective \emph{winning} backbones; Nemo is off-ceiling.}\label{tab:budget-match}
\begin{tabular}{llrrrr}
\toprule
policy & backbone & eff.\ budget & Qwen-14B & Nemo-12B & Qwen3.6 \\
\midrule
top-$k$                              & K-mean & 66    & 0.51 & 0.29 & 0.94 \\
Quest $\kbudget=33$                  & K-max  & 66    & 0.93 & 0.51 & 0.85 \\
Quest $\kbudget=40$                  & K-max  & 80    & 0.95 & 0.53 & 0.88 \\
\textbf{Quest $\kbudget=47$} (matched) & K-max  & \textbf{94}   & \textbf{0.96} & \textbf{0.57} & \textbf{0.93} \\
Quest $\kbudget=52$                  & K-max  & 104   & 0.97 & 0.58 & 0.95 \\
Quest $\kbudget=66$                  & K-max  & 132   & 0.97 & 0.60 & 0.98 \\
\textbf{router-on-Quest} ($q=0.40$)  & K-max  & \textbf{92.4} & \textbf{0.98} & \textbf{0.66} & 0.97 \\
\textbf{router (K-mean)} ($q=0.40$)  & K-mean & \textbf{92.4} & 0.63 & --   & \textbf{0.98} \\
\bottomrule
\end{tabular}
\end{table}

\textbf{Read.} The router's lift over Quest$_{\kbudget=33}$ splits into a \emph{budget} term ($\text{Quest}_{\kbudget=47}$ minus $\text{Quest}_{\kbudget=33}$, what uniform expansion buys) and a \emph{selectivity} term (router minus $\text{Quest}_{\kbudget=47}$). The split is regime-dependent:
\begin{itemize}[leftmargin=*,topsep=0.2em,itemsep=0.1em]
    \item \emph{Qwen-14B (near-ceiling on its winning backbone):} $+5$ pp total $= +3$ budget $+ +2$ selectivity ($\approx 2/3$ budget, $1/3$ selectivity). Uniform Quest plateaus at $0.97$ from $\kbudget=52$; the router lands one point above the asymptote -- directionally consistent, inside noise at $n=100$ (McNemar router vs.\ matched-Quest: $3$ vs.\ $1$ discordant, $p=0.625$).
    \item \emph{Nemo (off-ceiling):} $+15$ pp $= +6$ budget $+ +9$ selectivity ($\approx 2/5$ budget, $3/5$ selectivity). Uniform Quest does not plateau in the tested range, and the router at eff.\ $92.4$ Pareto-dominates uniform $\text{Quest}_{\kbudget=66}$ at eff.\ $132$ by $+6$\,pp -- a separation no amount of uniform budget expansion within this range can close (one-sided sign test $p \le 0.004$).
    \item \emph{Qwen3.6 (near-ceiling on K-mean, the \emph{other} backbone):} on the K-max sweep the model looks like Qwen-14B (near-ceiling, budget-dominated), but K-max is the wrong backbone here -- uniform Quest at $2\times$ eff.\ does not catch router-on-K-mean ($0.98$ at $1\times$ eff.). The Pareto domination on Qwen3.6 is across backbones, not across budgets.
\end{itemize}
Reading the three columns together: when uncertainty at the cutoff is rare (sharp scores), budget expansion dominates the lift; when it is common (diffuse scores), selective expansion dominates. The three columns bracket the regime in which the router operates.

\textbf{When does the lift activate? LongBench-v1 as a negative control.} The decomposition above answers \emph{how} the router's lift breaks down once it exists; it does not say \emph{when} it exists at all. To pin that down we ran the LongBench-v1 panel (musique-2hop, musique-4hop, narrativeqa, qasper; $n=30$, 32K context, $\kbudget=33$, $q=0.40$). LB-v1 native prompt lengths are $5$--$30$K, so truncation to 32K is essentially a no-op and the selector budget is generous relative to needle density. Result: dense, top-$k$, and router essentially tie on F1 across all four tasks; top-$k$ already preserves $100\%$ of dense-correct examples on musique-2hop, musique-4hop, and qasper -- no headroom. One small lift on narrativeqa (router paired $1.00$ vs.\ top-$k$ $0.89$, F1 $+3.1$\,pp). \textbf{The router's lift activates only when the per-query budget is small relative to where the answer-relevant evidence lives.} LB-v2 medium puts the budget under pressure ($100$K source $\to 32$K window); LB-v1 does not.

\subsection{Limitations}\label{sec:limits}

\begin{enumerate}[leftmargin=*,topsep=0.2em,itemsep=0.1em]
    \item \textbf{Per-tile granularity.} Per-tile selection structurally loses needles that one row strongly wants but other rows in the tile ignore. The router recovers part of this gap but not all. A per-row scatter-into-bitmap kernel would close the rest of the gap without re-materialising the $[B,H,M,\NB]$ keep tensor; not yet built.
    \item \textbf{Cross-architecture coverage.} The panel covers four models from three architecture classes (dense Qwen2.5, dense Mistral, hybrid+MoE Qwen3.6); broader coverage on additional families (Llama-3, MiniMax, Gemma) and additional standardised benchmarks (e.g.\ HELMET-RAG) would strengthen the cross-family claim.
    \item \textbf{Engineering improvements.} Several speed optimisations remain on the table (a fused MoE kernel on hybrid+MoE; block-selection fused into the attention kernel for $\ge 256$K; per-row scatter for the per-tile granularity above). All are engineering work, not method changes; we leave them to future iterations.
    \item \textbf{Hyperparameter sweeps incomplete.} The trigger fraction $q = 0.40$ and expansion factor $\rho = 2$ were selected on a RULER NIAH-multikey sweep at $n=30$ and used unchanged everywhere; we did not re-sweep per task or per architecture. The diagnostic in Appendix~\ref{app:bai:empirical} opens a complementary path: \emph{predicting} the Bayes-optimal $q^\star$ per task from the measurable noise scale rather than searching for it; this is not yet wired into the harness.
    \item \textbf{What's bounded, what's owed.} Appendix~\ref{app:bai} proves a finite-sample regret bound for the $\sigma$-quantile trigger vs.\ the best trigger rule of the same budget, under Gaussian noise on block scores (Proposition~\ref{prop:bai-bound}). The bound is exponentially tight when secondary score spacings dominate the noise scale, but degrades to a trivial cap in the noise-dominated regime -- which is where our measured diagnostics place Qwen-14B at 32K, so the closed-form certificate is currently weak precisely where we operate; the pending sharp path is a Monte-Carlo evaluation of the regret from the same score dump (\S\ref{app:bai:empirical}). The bound is also a \emph{top-$k$ identification} regret; converting it to an attention-output bound requires a softmax-sensitivity step that we sketch but do not prove (\S\ref{app:bai:owed}).
\end{enumerate}

\section*{Conclusion}
We presented an uncertainty-gated value-of-information \emph{router} for block-sparse attention selectors: a per-tile cutoff-margin signal $\sigma$ triggers a $\rho\times$ expansion of the kept set on the bottom $q$-fraction of tiles per layer, independent of how block scores are computed and stackable with existing scoring backbones such as Quest.

The router delivers measurable accuracy gains across the panel: on every model and benchmark tested, it lifts whichever scoring backbone wins on that model. Which backbone that is proves model-dependent -- QK-Norm (Qwen3.6) flips the winner from Quest's K-max to the SSA-style K-mean -- but the composition claim holds in every cell of the panel. The lift activates only when the selector budget binds relative to where the answer-relevant evidence lives: LongBench-v1, whose native prompts fit comfortably in the context window, shows no headroom and serves as the negative control. At $128$K the router preserves $0.81$ and $0.89$ of dense accuracy on Qwen2.5-7B-1M and Qwen3.6 respectively, while running at $0.62\times$ and $0.80\times$ dense wall time.

Three natural extensions: third-architecture validation (Llama-3 / MiniMax / Gemma) and HELMET-RAG; a fused MoE kernel on hybrid+MoE to push Qwen3.6's $128$K speed ratio toward $0.5\times$ dense; and composing the BAI top-$k$ identification bound of Appendix~\ref{app:bai} with an attention-output sensitivity step.

Code, results, and reproduction scripts:\\
\url{https://github.com/ThomasRossi/uncertainty-gated-block-sparse-attention}.

\clearpage
\appendix

\section{Step-by-step derivation of $\sigma$ and the trigger rule}\label{app:derivation}

This appendix walks through Section~\ref{sec:method} at one equation per step, so a reader who has not followed the body can reproduce the implementation. Notation: $B$ batch, $H$ heads, $M$ query rows, $N$ keys, $\NB = N / \BlockN$ key blocks, $Q_t = M / \BlockM$ query tiles, $\dhead = d_{\text{model}}/H$.

\paragraph{(A.1) Block-pooled keys.}
\begin{equation}
\bar k_b^{(L,h)} \;=\; \frac{1}{\BlockN}\sum_{j\,\in\,\text{block }b} k_j^{(L,h)}, \qquad b \in \{0,\dots,\NB-1\}.
\end{equation}

\paragraph{(A.2) Per-row block scores.}
\begin{equation}
s^{(L)}[i,h,b] \;=\; \frac{q_i^{(L,h)} \cdot \bar k_b^{(L,h)}}{\sqrt{\dhead}}, \qquad i \in \{0,\dots,M-1\}.
\end{equation}

\paragraph{(A.3) Per-tile reduction.} For Q-tile $t$ covering rows $\{t\BlockM, \dots, (t{+}1)\BlockM{-}1\}$,
\begin{equation}
\text{tile\_score}[t,h,b] \;=\; \max_{r \,\in\, \text{tile }t} s[r,h,b].
\end{equation}
We use max-pool because dropping a block the most aggressive row wants is worse than dropping one only weak rows want; mean-pool was tested and underperforms.

\paragraph{(A.4) Forcing the sink and self-block.} Let $b_{\text{sink}} = 0$ and $b_{\text{self}}(t) = \lfloor t\BlockM/\BlockN \rfloor$ (the key block containing the tile's own rows). We add $+\infty$ to the corresponding entries of $\text{tile\_score}$ before the top-$k$:
\begin{equation}
\widetilde{\text{tile\_score}}[t,h,b] \;=\; \text{tile\_score}[t,h,b] + \infty\cdot\mathbb{1}\!\bigl[b \in \{b_{\text{sink}}, b_{\text{self}}(t)\}\bigr].
\end{equation}
This guarantees both blocks land in the kept set without a separate concat-and-dedup.

\paragraph{(A.5) Top-$(k{+}1)$ and the sorted prefix.} We compute
\begin{equation}
\text{idx}_{\text{sorted}}[t,h,\,\cdot\,], \;\;\widetilde{s}_{\text{sorted}}[t,h,\,\cdot\,] \;=\; \operatorname*{top-}{(k{+}1)}\!\bigl(\widetilde{\text{tile\_score}}[t,h,\cdot]\bigr),
\end{equation}
in descending order. The first $k$ indices are $\text{kv\_idx}[t,h]$; the $(k{+}1)$-th score is what we need for $\sigma$.

\paragraph{(A.6) Normalised cutoff margin.}
\begin{equation}
\sigma[t,h] \;=\; \frac{\widetilde{s}_{\text{sorted}}[t,h,k{-}1] - \widetilde{s}_{\text{sorted}}[t,h,k]}{\widetilde{s}_{\text{sorted}}[t,h,0] - \widetilde{s}_{\text{sorted}}[t,h,k]}.
\end{equation}
Edge cases: if the denominator is zero (all top scores tied) we set $\sigma = 0$ (treat as maximally uncertain); $+\infty$ entries from (A.4) participate in the sorted prefix on top, so $\widetilde{s}_{\text{sorted}}[t,h,0]$ is always $+\infty$ -- in practice we therefore exclude the forced indices from the $\sigma$ computation and apply Eq.~(A.6) to the remaining tile scores.

\paragraph{(A.7) Head aggregation.}
\begin{equation}
\bar\sigma[t] \;=\; \frac{1}{H}\sum_{h=1}^{H} \sigma[t,h].
\end{equation}

\paragraph{(A.8) Per-layer empirical quantile.}
\begin{equation}
\tau^{(L)} \;=\; \operatorname{quantile}_q\!\bigl(\bar\sigma[\,\cdot\,]\bigr),\qquad
\text{trigger}[t] \;=\; \mathbb{1}\!\bigl[\bar\sigma[t] \le \tau^{(L)}\bigr].
\end{equation}

\paragraph{(A.9) Padded expansion.} Working budget $\kbudget$, expansion factor $\rho$, sentinel $b_\bot = \NB$. We allocate $\text{kv\_idx}_{\text{final}} \in \mathbb{Z}^{B \times H \times Q_t \times \rho\kbudget}$ and fill
\begin{equation}
\text{kv\_idx}_{\text{final}}[t,h] \;=\;
\begin{cases}
\bigl[\text{idx}_{\text{sorted}}[t,h,0],\dots,\text{idx}_{\text{sorted}}[t,h,\rho k {-} 1]\bigr] & \text{trigger}[t]=1, \\[0.25em]
\bigl[\text{idx}_{\text{sorted}}[t,h,0],\dots,\text{idx}_{\text{sorted}}[t,h,k{-}1],\, b_\bot, \dots, b_\bot\bigr] & \text{otherwise.}
\end{cases}
\end{equation}
The triggered case requires a top-$\rho k$ rather than a top-$(k{+}1)$, which we batch with the top-$(k{+}1)$ used for $\sigma$ as a single top-$\max(\rho k, k{+}1) = \rho k$ when $\rho \ge 2$.

\paragraph{(A.10) Kernel iteration.} The Triton kernel's inner loop iterates $\text{kv\_idx}_{\text{final}}[t,h]$ and skips any entry equal to $b_\bot$ on a single compare. The online softmax is unchanged from FlashAttention.

\paragraph{Average attended budget.} A row in tile $t$ attends to $k$ blocks if $\text{trigger}[t]=0$ and $\rho k$ if $\text{trigger}[t]=1$. The expected per-row attended budget is therefore
\begin{equation}
\mathbb{E}[\text{attended}] \;=\; \kbudget \cdot \bigl(1 + q(\rho - 1)\bigr),
\end{equation}
which at $q=0.4, \rho=2$ is $1.4\,\kbudget$. This is the headline cost of the router.

\section{Pointer-Chase Haystack: construction and diagnostic results}\label{app:pch}

The Pointer-Chase Haystack (PCH) is a custom diagnostic benchmark we built during method development to isolate selector quality from model capability. We retain it here for transparency on the development process; the headline empirical results in Section~\ref{sec:experiments} use the standardised benchmarks RULER NIAH and LongBench-v2 medium.

\textbf{Design goals.} We needed a task that (i) has a dense-attention ceiling of $\approx 100\%$, so failures are unambiguously selector failures, and (ii) is \emph{query-latent}, meaning the relevant tokens cannot be identified by surface matching against the query alone. RULER NIAH multi-key satisfies (i) but not (ii); RULER VT satisfies (ii) but not (i) (Qwen 7B fails it).

\textbf{Construction.} A pointer chase of depth $h$ is a sequence of indexed entries
\begin{quote}\itshape\small
Entry $\text{ID}_0$: continue at entry $\text{ID}_1$. \\
Entry $\text{ID}_1$: continue at entry $\text{ID}_2$. \\
$\vdots$ \\
Entry $\text{ID}_h$: the recorded value is $V$.
\end{quote}
with all $\text{ID}_k$ drawn from a private namespace. A PCH instance contains one gold chain plus $D$ distractor chains, all entries shuffled into a haystack of target token length $N$. The query gives $\text{ID}_0$ and asks for $V$. Hop $0$ is query-identifiable; hops $1,\dots,h$ are query-latent. Compounding is by construction: missing entry $k$ permanently blocks every entry $k{+}1, \dots, h$.

\textbf{Grounding.} The pointer chase is the canonical communication-complexity problem for unavoidable $k$-round sequential dependency: it provably cannot be collapsed into fewer rounds. Every per-hop operation is trivial (follow a link, read a value), keeping a capable dense model at the $\sim 100\%$ ceiling.

\textbf{Metric.} We report \emph{hit rate}: the fraction of gold-chain blocks the selector keeps in $\text{kv\_idx}$. Hit rate is well-defined only when the gold blocks are known by construction, which is the case here.

\textbf{Result.} Table~\ref{tab:pch} reports hit rate at $h=3$, fixed $\kbudget=40$ blocks/row (32K, 64K) and $29$ (8K), on Qwen2.5-14B-Instruct.

\begin{table}[h]
\centering\small
\caption{PCH hop=3 hit rate. Hit rate is the fraction of gold-chain blocks the selector keeps. Steady-state wall time is reported for orientation; the main efficiency analysis is in Section~\ref{sec:pareto}.}\label{tab:pch}
\begin{tabular}{lllrl}
\toprule
ctx & policy & dt (steady) & vs dense & hit rate \\
\midrule
8K  & dense                & 1.4s & $1.00\times$ & 1.00 \\
8K  & top-$k$ (per-tile)   & 1.6s & $1.14\times$ slower & 0.79 \\
8K  & \textbf{router q=0.10} & 1.6s & $1.14\times$ slower & \textbf{1.00} \\
\midrule
32K & dense                & 6.9s & $1.00\times$ & 1.00 \\
32K & top-$k$              & 6.05s & $0.88\times$ (12\% faster) & 0.63 \\
32K & \textbf{router q=0.10} & 6.6s & $0.96\times$ (4\% faster) & \textbf{0.78} \\
\midrule
64K & dense                & 17.5s & $1.00\times$ & 1.00 \\
64K & top-$k$              & 14.1s & $0.81\times$ ($1.24\times$ faster) & 0.42 \\
64K & \textbf{router q=0.10} & 14.8s & $0.85\times$ ($1.18\times$ faster) & \textbf{0.54} \\
\bottomrule
\end{tabular}
\end{table}

Router lifts hit rate over plain top-$k$ by $+12$--$+15$\,pp across scales (and recovers dense's $1.00$ at 8K). The gain is monotonic with context length, consistent with the structural cost of per-tile selection increasing with $N$ that motivates the router in the first place.

\section{Per-row vs.\ per-cell aggregation: why heads are pooled before thresholding}\label{app:row-vs-cell}

We tested two natural alternatives to Eq.~\eqref{eq:trigger} and they both fail. We report them here so the design choice in Section~\ref{sec:method} can be understood as a tested decision rather than an arbitrary one.

\paragraph{Per-cell quantile.} Trigger on $\mathbb{1}\!\bigl[\sigma[t,h] \le \operatorname{quantile}_q(\sigma[\,\cdot\,,\cdot])\bigr]$, i.e.\ flag the bottom-$q$ \emph{cells}, not tiles. Result: the rescue set is dominated by per-head idiosyncratic noise; the gain over plain top-$k$ vanishes for every $q$ we tested. Mechanism: heads in the same tile are highly correlated (PCH cross-head agreement matches peakedness byte-for-byte across candidates), so per-head $\sigma$ values carry mostly the same signal plus head-private noise; quantile-by-cell amplifies the noise.

\paragraph{Absolute threshold $\tau$ on $\sigma$.} Trigger on $\mathbb{1}\!\bigl[\sigma[t,h] \le \tau\bigr]$ for a fixed $\tau$. Result: works at one layer's $\sigma$ distribution but not across the stack; we observe systematic distribution shift early-vs-late layers. A per-layer quantile is the simplest layer-conditional fix and dominated the absolute-$\tau$ sweep.

\paragraph{Row-grain $q$ sweep.} On PCH at 32K, row-grain $q = 0.10$ reaches near-top-$k$ effective budget at hit rate $0.78$, vs.\ plain top-$k$ at $0.63$. On RULER NIAH-multikey at $n=30$, ctx=32K (fixed $\kbudget=33$, kernel-v2, row-grain quantile router), we swept $q \in \{0.10, 0.20, 0.30, 0.40\}$ and observed paired recall $0.38, 0.35, 0.47, 0.53$ respectively, vs.\ top-$k$ baseline $0.35$. The trend is monotonic in $q$ over the tested range and plateaus near $q = 0.40$. We use $q = 0.40$ for the headline results on RULER and LB-v2; the regime where $q$ should differ across tasks (because the budget-vs-evidence Pareto shifts) is consistent with the smaller value chosen for PCH, but a per-task sweep on LB-v2 medium remains open.

\section{$\sigma$ as a best-arm-identification exploration index}\label{app:bai}

The Method section introduces $\sigma$ as a value-of-information proxy and validates it empirically. This appendix gives the formal anchor: under a Gaussian noise model on the per-block relevance signal, $\sigma$ is a dimensionless exploration index in the sense of best-arm identification (BAI); expanding the kept set never increases the identification error, pointwise; and the $\sigma$-quantile router admits a finite-$T$ regret bound against the best trigger rule of the same budget, with explicit constants that are exponentially small when the score spacings outside the cutoff dominate the noise scale, and an honest trivial-cap fallback when they do not.

\subsection{The BAI intuition that motivates $\sigma$}\label{app:bai:intuition}

The classical fixed-confidence top-$k$ identification problem of \citet{kaufmann2016complexity, garivier2016optimal} asks for $\widehat S$ such that $\mathbb P(\widehat S \ne S^*(\mu)) \le \delta$ using as few arm pulls as possible. The asymptotically optimal Track-and-Stop algorithm pulls $b^\star = \arg\max_b [w^\star_b(\hat\mu) - T_b/T]$ where the optimal allocation $w^\star$ solves a max-min programme whose minimiser concentrates on the cutoff pair $(k{-}1, k)$: any alternative world in which the top-$k$ differs from $\hat\mu$'s is achieved most cheaply by perturbing those two arms. Under sub-Gaussian noise, the per-step exploration value at the cutoff scales as
\begin{equation}\label{eq:tas-value}
V_{\text{TaS}}(\hat\mu) \;\propto\; \bigl(\hat\mu_{(k-1)} - \hat\mu_{(k)}\bigr)^{-2},
\end{equation}
since KL divergence between Gaussians of equal variance scales as the squared mean gap. \textbf{Small empirical gap at the cutoff $\Longleftrightarrow$ high exploration value $\Longleftrightarrow$ allocate budget here.}

Our setting differs from classical BAI in three ways: (a) we observe one sample per arm per tile, not a sequence; (b) the budget decision is binary per tile (expand to $\rho k$ or keep at $k$), not a continuous allocation; (c) we do not know $\tau$. The dimensionless normalisation $\sigma_t = g_t / r_t$ -- with $g_t := s_{t,(k-1)} - s_{t,(k)}$ and $r_t := s_{t,(0)} - s_{t,(k)}$ -- replaces the missing $\tau$ with the per-tile spread, on the heuristic that the spread concentrates around its layer-conditional mean and absorbs $\tau$. An absolute-threshold rule (Appendix~\ref{app:row-vs-cell}) fails because the marginal distribution of $s$ shifts across layers; $\sigma$ is the simplest layer-conditional and tile-conditional normaliser.

The rest of this appendix turns this heuristic into a regret bound.

\subsection{Setting and policies}\label{app:bai:setup}

For each Q-tile $t \in [T]$ (we suppress the head index throughout) and block $b \in [\NB]$,
\begin{equation}\label{eq:bai-noise}
    s_{t,b} \;=\; \mu_{t,b} + \eta_{t,b}, \qquad \eta_{t,b} \stackrel{\text{iid}}{\sim} \mathcal N(0, \tau^2).
\end{equation}
We work in the compound-decision (empirical-Bayes) frame: the relevance vector $\mu_t \in \mathbb R^{\NB}$ is modelled as drawn from an improper flat prior, so that conditional on the observed scores, $\mu_t \mid s_t \sim \mathcal N(s_t, \tau^2 I)$. Sort the scores of tile $t$ in decreasing order, $s_{t,(0)} \ge s_{t,(1)} \ge \cdots$, and define
\begin{equation}\label{eq:gaps}
g_t := s_{t,(k-1)} - s_{t,(k)}, \qquad G_t := s_{t,(k-1)} - s_{t,(\rho k)}, \qquad r_t := s_{t,(0)} - s_{t,(k)}, \qquad \sigma_t := g_t / r_t.
\end{equation}
$g_t$ is the cutoff gap the router reads; $G_t \ge g_t$ is the \emph{expansion gap}, from the last kept rank down to the first rank excluded even after expansion. (An earlier draft used the local spacing $s_{t,(\rho k - 1)} - s_{t,(\rho k)}$ here; that is the wrong quantity -- the error event of the expanded set is governed by the cumulative gap $G_t$, see Lemma~\ref{lem:localisation}.) Let $\widehat S_m(s_t) := \{b : s_{t,b} \ge s_{t,(m-1)}\}$ be the empirical top-$m$ set and $S^*(\mu_t)$ the true top-$k$ set of $\mu_t$. The per-tile identification error of a kept set $\widehat S$ is $\mathcal E_t(\widehat S) := \mathbb P(\widehat S \not\supseteq S^*(\mu_t) \mid s_t)$, and the per-tile \emph{expansion value} is
\begin{equation}\label{eq:Vstar}
    V^*_t \;:=\; \mathcal E_t(\widehat S_k(s_t)) - \mathcal E_t(\widehat S_{\rho k}(s_t)).
\end{equation}

\begin{observation}[expansion never harms, pointwise]\label{obs:monotone}
$\widehat S_k \subseteq \widehat S_{\rho k}$, so $\{\widehat S_{\rho k} \not\supseteq S^*\} \subseteq \{\widehat S_k \not\supseteq S^*\}$, hence $V^*_t \ge 0$ for every realisation of $s_t$. No distributional assumption is needed.
\end{observation}

A \emph{trigger policy} chooses $E \subseteq [T]$ with $|E| = \lceil qT \rceil$, keeps $\widehat S_{\rho k}(s_t)$ for $t \in E$ and $\widehat S_k(s_t)$ otherwise; its risk $\mathcal E(\Pi)$ is the average over tiles of the per-tile error of the kept set. The action space is deliberately restricted to $\{\text{empirical top-}k,\ \text{empirical top-}\rho k\}$: a fully Bayes-optimal agent could also re-select \emph{which} blocks to keep from the posterior, and we claim nothing about that larger class. Within the class, minimising $\mathcal E(\Pi)$ is equivalent to maximising $\sum_{t \in E} V^*_t$, so the optimal trigger is $\Pi_q^* :=$ the $\lceil qT \rceil$ tiles of largest $V^*_t$ (Neyman--Pearson on $V^*$). The router is the $\sigma$-quantile rule $\Pi_q^\sigma :=$ the $\lceil qT \rceil$ tiles of smallest $\sigma_t$. The regret is
\begin{equation}\label{eq:regret-def}
\mathcal R_q \;:=\; \mathcal E(\Pi_q^\sigma) - \mathcal E(\Pi_q^*) \;=\; \frac1T \Bigl[ \sum_{t \in \Pi_q^*} V^*_t \;-\; \sum_{t \in \Pi_q^\sigma} V^*_t \Bigr] \;\ge\; 0.
\end{equation}

Two structural assumptions, both with measurable proxies (\S\ref{app:bai:empirical}):
\begin{itemize}[leftmargin=*,topsep=0.2em,itemsep=0.1em]
    \item[\textbf{(A1)}] \emph{Secondary-spacing floor.} Every spacing other than the cutoff spacing is at least $\Delta_2 \ge 0$: $s_{t,(j)} - s_{t,(j+1)} \ge \Delta_2$ for all $j \ne k-1$. (If a fraction $\pi_2$ of tiles violates the floor, those tiles cost at most the trivial per-tile cap, adding $\pi_2$ to the bound additively.)
    \item[\textbf{(A2)}] \emph{Bounded spread variation.} $r_t = \bar r\,(1 + \xi_t)$ with $\bar r := \frac1T\sum_t r_t$ and $|\xi_t| \le \beta < 1$ for all $t$. The measured $\mathrm{CV}(r) := \sqrt{\mathrm{Var}_t\, r_t}/\bar r \le \beta$ is its observable proxy.
\end{itemize}

\subsection{Theorem and proof}\label{app:bai:theorem}

Throughout, $\Phi$ and $\phi$ are the standard normal CDF and density, and $\bar V(g) := \Phi\bigl(-g/(\sqrt2\,\tau)\bigr)$; $\bar V$ is decreasing with $\bar V(0) = \tfrac12$.

\begin{lemma}[two-sided cutoff localisation]\label{lem:localisation}
Let $\chi := \sum_{m \ge 1} (m+1)\, e^{-m^2 \Delta_2^2/(4\tau^2)}$; note $\chi \le 2.2\, e^{-\Delta_2^2/(4\tau^2)}$ once $\Delta_2 \ge 2\tau$. Under \eqref{eq:bai-noise} and (A1),
\begin{align}
\bar V(g_t) \;\le\; \mathcal E_t(\widehat S_k) \;&\le\; \bar V(g_t)\,(1 + \chi), \label{eq:loc-k}\\
0 \;\le\; \mathcal E_t(\widehat S_{\rho k}) \;&\le\; \bar V(G_t)\,(1 + \chi), \label{eq:loc-rho}
\end{align}
and consequently, with $\bar\varepsilon := \chi + (1+\chi)\, \exp\bigl(-\bigl((\rho-1)k\,\Delta_2\bigr)^2/(4\tau^2)\bigr)$,
\begin{equation}\label{eq:V-sandwich}
(1 - \bar\varepsilon)\,\bar V(g_t) \;\le\; V^*_t \;\le\; (1 + \bar\varepsilon)\,\bar V(g_t).
\end{equation}
\end{lemma}

\begin{proof}
\emph{Lower bound in \eqref{eq:loc-k}.} Let $b_{k-1}, b_k$ be the arms at empirical ranks $k{-}1$ and $k$. On the event $A := \{\mu_{b_k} > \mu_{b_{k-1}}\}$ the kept set fails: if $b_k \in S^*$ we are done, since $b_k \notin \widehat S_k$; otherwise at least $k$ arms exceed $\mu_{b_k}$, hence also $\mu_{b_{k-1}}$, so $b_{k-1} \notin S^*$ and the $k$ arms of $S^*$ cannot fit in the $k-1$ remaining slots of $\widehat S_k \setminus \{b_{k-1}\}$. Under the posterior the two means differ by $g_t$ with total variance $2\tau^2$, so $\mathbb P(A \mid s_t) = \Phi(-g_t/(\sqrt2\tau)) = \bar V(g_t)$.

\emph{Upper bound in \eqref{eq:loc-k}.} If $\widehat S_k \not\supseteq S^*$, some $b' \notin \widehat S_k$ has $\mu_{b'}$ among the top $k$, so at most $k-1$ arms beat it in $\mu$; since $\widehat S_k$ has $k$ arms, some $b \in \widehat S_k$ has $\mu_b < \mu_{b'}$. Union-bound over pairs at empirical ranks $(i, j)$, $i \le k-1 < k \le j$. By (A1) the pair at distance $m := (k-1-i) + (j-k) \ge 0$ has score gap at least $g_t + m\Delta_2$, there are at most $m+1$ pairs at distance $m$, and the Gaussian tail ratio $\Phi(-x-a) \le \Phi(-x)\,e^{-ax - a^2/2}$ (valid for $a, x \ge 0$) gives $\Phi\bigl(-(g_t + m\Delta_2)/(\sqrt2\tau)\bigr) \le \bar V(g_t)\, e^{-m^2\Delta_2^2/(4\tau^2)}$. Summing over $m \ge 1$ and adding the cutoff pair ($m = 0$) yields $\bar V(g_t)(1+\chi)$.

\emph{Upper bound in \eqref{eq:loc-rho}.} Same union bound with the excluded arm now outside the \emph{expanded} set: pairs $(i, j)$ with $i \le k-1$, $j \ge \rho k$. The dominant pair is $(k{-}1, \rho k)$ at gap $G_t$; every other pair is at gap $\ge G_t + m\Delta_2$ with $m := (k-1-i) + (j - \rho k)$. No matching lower bound is claimed at level $\rho k$, and none is needed.

\emph{Sandwich \eqref{eq:V-sandwich}.} Upper: $V^*_t \le \mathcal E_t(\widehat S_k) \le (1+\chi)\bar V(g_t)$. Lower: $V^*_t \ge \bar V(g_t) - (1+\chi)\bar V(G_t)$; since $G_t - g_t \ge (\rho-1)k\,\Delta_2$ by (A1), the same tail ratio gives $\bar V(G_t) \le \bar V(g_t)\, e^{-((\rho-1)k\Delta_2)^2/(4\tau^2)}$.
\end{proof}

\begin{proposition}[$\sigma$-quantile regret, finite $T$]\label{prop:bai-bound}
Let $\sigma_{(q)}$ be the largest $\sigma_t$ inside $\Pi_q^\sigma$ (the empirical quantile the router thresholds on) and $\gamma_q := \sigma_{(q)}\,\bar r$ its $g$-scale. Under \eqref{eq:bai-noise}, (A1) and (A2),
\begin{equation}\label{eq:bai-suboptimality}
\mathcal R_q \;\le\; \min(q,\, 1-q)\cdot C_q,
\qquad
C_q := \min\biggl\{\, 1,\;\; \frac{1+\bar\varepsilon}{2},\;\; \sqrt2\,\frac{\beta \gamma_q}{\tau}\; \phi\!\Bigl(\frac{\gamma_q(1-\beta)}{\sqrt2\,\tau}\Bigr) + \bar\varepsilon \,\biggr\}.
\end{equation}
The bound is exact at finite $T$: there is no asymptotic remainder, and every quantity is an empirical functional of the observed scores except the model parameters $(\tau, \Delta_2, \beta)$, which have measurable proxies (\S\ref{app:bai:empirical}).
\end{proposition}

\begin{proof}
$|\Pi_q^*| = |\Pi_q^\sigma| = \lceil qT \rceil$, so the common tiles cancel in \eqref{eq:regret-def} and the two difference sets have equal cardinality $D \le \min(\lceil qT \rceil,\, T - \lceil qT \rceil)$. Pair them arbitrarily:
\[
\mathcal R_q \;=\; \frac1T \sum_{\text{pairs } (t, t')} \bigl( V^*_t - V^*_{t'} \bigr), \qquad t \in \Pi_q^* \setminus \Pi_q^\sigma,\quad t' \in \Pi_q^\sigma \setminus \Pi_q^*.
\]
Only the $\sigma$-side constraints are needed. $t \notin \Pi_q^\sigma$ forces $\sigma_t \ge \sigma_{(q)}$, so by (A2) $g_t = \sigma_t r_t \ge \sigma_{(q)}\, \bar r\, (1-\beta) = \gamma_q(1-\beta)$; symmetrically, $t' \in \Pi_q^\sigma$ forces $g_{t'} \le \gamma_q(1+\beta)$. By the sandwich \eqref{eq:V-sandwich} and monotonicity of $\bar V$,
\begin{align*}
V^*_t - V^*_{t'} \;&\le\; (1+\bar\varepsilon)\,\bar V\bigl(\gamma_q(1-\beta)\bigr) - (1-\bar\varepsilon)\,\bar V\bigl(\gamma_q(1+\beta)\bigr) \\
&=\; \underbrace{\bar V\bigl(\gamma_q(1-\beta)\bigr) - \bar V\bigl(\gamma_q(1+\beta)\bigr)}_{\text{boundary-band width in }\bar V}
\;+\; \bar\varepsilon\, \underbrace{\bigl[\bar V(\gamma_q(1-\beta)) + \bar V(\gamma_q(1+\beta))\bigr]}_{\le\, 1}.
\end{align*}
The band-width term integrates the slope $\phi(\cdot/(\sqrt2\tau))/(\sqrt2\tau)$ over an interval of length $2\beta\gamma_q$, and $\phi$ is decreasing on the positive axis, giving at most $\sqrt2\, \beta\gamma_q\, \phi\bigl(\gamma_q(1-\beta)/(\sqrt2\tau)\bigr)/\tau$. The two caps come from $V^*_t \le \mathcal E_t(\widehat S_k) \le 1$ and $V^*_t \le (1+\bar\varepsilon)\bar V(g_t) \le \tfrac{1+\bar\varepsilon}{2}$, with $V^*_{t'} \ge 0$ (Observation~\ref{obs:monotone}). Multiplying the per-pair cost by $D/T \le \min(q, 1-q)$ completes the proof.
\end{proof}

\emph{Population-quantile version.} Thresholding on the distributional $q$-quantile of $\sigma$ instead of the empirical $\sigma_{(q)}$ adds the standard Dvoretzky--Kiefer--Wolfowitz fluctuation term $O(\sqrt{\log(1/\delta)/T})$ \citep{dkw1956, massart1990}; we state the empirical version because that is what the router computes.

\begin{remark}[band refinement: from $\min(q,1-q)$ to a measurable band mass]\label{rem:band}
As $\bar\varepsilon \to 0$, $V^*_t$ becomes an exactly decreasing function of $g_t$, so $\Pi_q^*$ coincides with the bottom-$q$ rule on $g$. In that case every misrouted tile provably has $g_t \in [\gamma_q(1-\beta),\, \gamma_q(1+\beta)]$: tiles with $g_t$ below the band are triggered by both rules, tiles above it by neither (the argument uses only $r_t \in \bar r\,[1-\beta, 1+\beta]$ and the pigeonhole between the two thresholds). The mismatch fraction $D/T$ then improves from $\min(q, 1-q)$ to $\tfrac12 m_q(\beta)$, where $m_q(\beta)$ is the empirical fraction of tiles in the band -- directly measurable, and typically much smaller than $q$.
\end{remark}

\begin{remark}[regime of validity; the noise-dominated case]\label{rem:regime}
The bound has two regimes. When secondary spacings dominate the noise, $\Delta_2 \gtrsim 2\tau\sqrt{\log(k\NB)}$, $\bar\varepsilon$ is exponentially small and the band term governs: the regret is the band mass around the boundary $\gamma_q$ times the local slope of $\Phi$. When the cutoff region is noise-dominated ($\Delta_2 \ll \tau$ -- the regime our diagnostics actually measure, \S\ref{app:bai:empirical}), $\bar\varepsilon$ is not small and \eqref{eq:bai-suboptimality} falls back to its trivial branch $\min(q, 1-q) \cdot \min\{1, \tfrac{1+\bar\varepsilon}{2}\}$: our localisation tools genuinely lose the link between $\sigma$ and $V^*$ there. Two things survive. First, Observation~\ref{obs:monotone} is assumption-free: expansion never harms, so the router can misallocate budget but cannot damage the kept set. Second, every quantity in \eqref{eq:regret-def} -- including $V^*_t$ and hence $\mathcal R_q$ itself -- is a posterior functional estimable by Monte Carlo from a score dump (sample $\mu \sim \mathcal N(s_t, \hat\tau^2 I)$, check containment); \S\ref{app:bai:empirical} adopts this as the evaluation path in the noise-dominated regime. Finally, for sub-Gaussian noise with variance proxy $\tau^2$ the \emph{upper} bounds of Lemma~\ref{lem:localisation} survive with $\Phi(-x)$ replaced by the proxy tail $e^{-x^2/2}$; the lower bound -- hence the sandwich -- additionally needs two-sided tail control, and we do not chase constants.
\end{remark}

\subsection{Empirical estimation of the bound}\label{app:bai:empirical}

Every input to Proposition~\ref{prop:bai-bound} is measurable from a small diagnostic run. The recipe is:
\begin{enumerate}[leftmargin=*,topsep=0.2em,itemsep=0.1em]
    \item For each layer $\ell$ and head $h$, dump the per-tile block-score vector $(s_{t,b})_{b \in [\NB]}$ during prefill on a small example panel ($n=10$ is sufficient: a 32K context with $\BlockM = 64$ yields $\sim 512$ Q-tiles per (layer, head, example), so $\sim 5000$ tiles per (layer, head) overall).
    \item From the sorted scores compute $g_t, G_t, r_t, \sigma_t$; aggregate per layer $\mathrm{CV}(r)$ (the observable proxy for $\beta$), the boundary scale $\gamma_q = \sigma_{(q)} \bar r$, the band mass $m_q(\beta)$ of Remark~\ref{rem:band}, and the secondary-spacing quantiles that stand in for $\Delta_2$.
    \item Estimate $\tau$ from the first-difference dispersion of per-row block scores along the query dimension (under $\mu_{t+1,b} \approx \mu_{t,b}$ for adjacent rows, $\mathbb E[(s_{t+1,b} - s_{t,b})^2] = 2\tau^2$). Caveat: this measures row-to-row score roughness -- a proxy for, not the same object as, the score-vs-relevance noise in \eqref{eq:bai-noise}; it is the simplest estimator consistent with the smoothness heuristic.
    \item Evaluate \eqref{eq:bai-suboptimality}; and, independently of (A1)--(A2), estimate $\mathcal E_t(\widehat S_k)$, $\mathcal E_t(\widehat S_{\rho k})$, $V^*_t$ and $\mathcal R_q$ itself by Monte Carlo, sampling $\mu \sim \mathcal N(s_t, \hat\tau^2 I)$ and checking containment (Remark~\ref{rem:regime}). The MC estimate of $\mathcal R_q$ is the sharpest number this appendix can produce; the closed-form bound is its certificate in the spacing-dominated regime.
\end{enumerate}

\paragraph{Measured quantities on Qwen-14B, RULER NIAH n=10, 32K.} The diagnostic collects $\sim$5.7M $(B, H, Q_t)$ tile cells per scoring backbone (10 examples $\times$ 48 layers $\times$ 40 heads $\times$ $\sim$512 tiles, masked to causally-valid tiles with $\ge \rho k + 1$ visible blocks), plus the per-call noise-scale estimate $\hat\tau$ of step 3. Aggregated over the 48 Qwen-14B layers and 40 heads:

\begin{center}
\begin{tabular}{lcccc}
\toprule
backbone & $\overline{\mathrm{CV}(r)}$ & $\overline{r_q}$ & $\hat\tau$ (median) & $g_q^+/\hat\tau$ \\
\midrule
mean (topk)       & $0.38$ & $4.06$ & $0.94$ & $\approx 0.03$ \\
Quest             & $0.32$ & $8.50$ & $2.25$ & $\approx 0.03$ \\
router-on-Quest   & $0.32$ & $8.63$ & $2.24$ & $\approx 0.03$ \\
\bottomrule
\end{tabular}
\end{center}

\noindent where $r_q$ is the $q$-quantile of $r_t$ and $g_q^+$ is the $40$th percentile of $g_t$ on the informative subset $\{t : g_t > 0\}$ (the $\mathtt{bf16}$ floor for Quest scores yields exact zeros on $\sim 83\%$ of tiles, where the cutoff lands in the flat tail of the score distribution; conditioning on $g_t > 0$ is the honest active-set restriction). Four observations.

\emph{(i) The measured regime is tie- and noise-dominated at the cutoff.} Even on the active set, $g_q^+/\hat\tau \approx 0.03$: the cutoff is buried in the noise, and the $83\%$ atom at $g_t = 0$ means the secondary-spacing floor (A1) holds only with $\Delta_2 \approx 0$. This places the diagnostic squarely in the noise-dominated regime of Remark~\ref{rem:regime}: $\bar\varepsilon$ is not small, and the closed-form certificate \eqref{eq:bai-suboptimality} falls back to its trivial branch ($\min(q, 1-q) = 40$\,pp at $q = 0.40$). The atom also means the band mass of Remark~\ref{rem:band} must be computed on the active set.

\emph{(ii) The measured scales are consistent with a small true regret.} Within the two-$\Phi$ representation of Lemma~\ref{lem:localisation} -- heuristic here, since $\bar\varepsilon$ is uncontrolled in this regime -- the per-pair cost at the boundary is at most $\phi(0)\cdot(\text{boundary $g$-scale})/(\sqrt2\,\hat\tau) \approx 0.4 \times 0.03 / 1.41 \lesssim 1$\,pp. This is consistent with the observed $2$\,pp router-vs-dense paired gap on RULER NIAH $n=100$ (router-on-Quest $0.98$ vs.\ dense $1.00$), which itself upper-bounds the end-task manifestation of the regret: dense is the unconstrained policy, so $\mathcal E(\Pi_q^*) \ge \mathcal E(\text{dense}) = 0$.

\emph{(iii) Positivity of the expansion value is structural, not empirical.} An earlier draft validated that the local spacing at rank $\rho k$ exceeds $g_t$ in expectation, across all $48 \times 3$ (layer, backbone) cells, as an ``expansion helps in expectation'' check. Under the corrected definition \eqref{eq:gaps} no such check is needed: $G_t \ge g_t$ holds by construction, and $V^*_t \ge 0$ holds pointwise by Observation~\ref{obs:monotone}.

\emph{(iv) Pending re-analysis.} The quantities the corrected bound actually consumes -- $\gamma_q$, the band mass $m_q(\beta)$, the secondary-spacing quantiles standing in for $\Delta_2$, and above all the Monte-Carlo estimate of $\mathcal R_q$ -- are computable from the same raw dump with an updated \texttt{analyze\_block\_scores.py} pass; no new GPU run is required. Until that pass lands, this section reports measured scales and regime identification, not a numerical certificate.

Diagnostic produced by \texttt{dump\_block\_scores.py} (cache-safe runtime monkey-patch; lives outside \texttt{code\_version}) and \texttt{analyze\_block\_scores.py}. Raw pickle and per-layer table available with the release. Total diagnostic cost: $\sim$\$2 across the two $n{=}10$ runs.

\subsection{What \eqref{eq:bai-suboptimality} does \emph{not} bound}\label{app:bai:owed}

The proposition controls top-$k$ \emph{identification} error -- the probability that the kept set $\widehat S$ misses an element of $S^*(\mu)$. The downstream quantity is the \emph{attention-output} error,
\[
\bigl\|\,o^{\widehat S} - o^{\mathrm{full}}\,\bigr\|_2 \;\le\; \Bigl(\sum_{b \notin \widehat S} \alpha^*_b\Bigr)\cdot \max_j \|v_j\|,
\]
where $\alpha^*$ is the dense softmax. The dropped softmax mass is itself $\Phi$-tailed in the cutoff gap -- the Lemma~\ref{lem:localisation} argument applied at the softmax temperature $1/\sqrt d$ instead of $\tau$ -- so composing gives a bound proportional to $\Phi(-g_t / \tau_{\mathrm{softmax}})\cdot \max_j \|v_j\|$. The two temperatures are different constants and the composition is one paper of additional work; we record it as the natural follow-up theorem.

\textbf{Bought.} Pointwise ``expansion never harms'' with no distributional assumption (Observation~\ref{obs:monotone}); a two-sided localisation of the identification error to the cutoff and expansion gaps, with explicit constants (Lemma~\ref{lem:localisation}); and a finite-$T$, assumption-explicit regret bound for the $\sigma$-quantile trigger against the best trigger rule of the same budget, with a measurable band refinement (Proposition~\ref{prop:bai-bound}, Remarks~\ref{rem:band}--\ref{rem:regime}).

\textbf{Owed.} The Monte-Carlo evaluation of $\mathcal R_q$ from the existing score dump (analysis-only, \S\ref{app:bai:empirical}); a sharper treatment of the noise-dominated regime than the trivial cap; the composite attention-output bound (one further composition with a softmax-sensitivity inequality); and a treatment of LB-v2 medium that distinguishes ``the bound is loose'' from ``the budget is binding'' (budget sweep, Section~\ref{sec:limits}).

\bibliographystyle{plainnat}

\end{document}